\def\year{2021}\relax
\documentclass[letterpaper]{article} 
\usepackage{aaai21}  
\usepackage{times}  
\usepackage{helvet} 
\usepackage{courier}  
\usepackage[hyphens]{url}  
\usepackage{graphicx} 
\usepackage{natbib}  
\usepackage{caption} 
\usepackage{url}
\usepackage{ntheorem}
\usepackage{amsfonts}
\usepackage{amsmath}
\usepackage{booktabs}
\newtheorem{lemma}{Lemma}
\newtheorem{theorem}{Theorem}
\newtheorem*{proof}{Proof}
\newtheorem{corollary}{Corollary}
\usepackage{multirow}
\usepackage{subfigure}
\usepackage{amssymb}
\usepackage{bbding}
\usepackage[switch]{lineno}

\title{Revisiting Graph Convolutional Network on Semi-Supervised Node Classification from an Optimization Perspective}
\author{
    Hongwei Zhang\textsuperscript{\rm 1},
    Tijin Yan\textsuperscript{\rm 1},
    Zejun Xie\textsuperscript{\rm 2},
    Yuanqing Xia\textsuperscript{\rm 1}*,
   Yuan Zhang\textsuperscript{\rm 1},
    \\
}
\affiliations{

    \textsuperscript{\rm 1}Beijing Institute of Technology, Beijing, China\\
    \textsuperscript{\rm 2}Renmin University of China, Beijing, China\\
}
\nocopyright
\begin{document}
\maketitle
\pdfoutput=1
\begin{abstract}

Graph convolutional networks (GCNs) have achieved promising performance on various graph-based tasks. However they suffer from over-smoothing when stacking more layers. In this paper, we present a quantitative study on this observation and develop novel insights towards the deeper GCN. First, we interpret the current graph convolutional operations from an optimization perspective and argue that over-smoothing is mainly caused by the naive first-order approximation of the solution to the optimization problem. Subsequently, we introduce two metrics to measure the over-smoothing on node-level tasks. Specifically, we calculate the fraction of the pairwise distance between connected and disconnected nodes to the overall distance respectively. Based on our theoretical and empirical analysis, we establish a universal theoretical framework of GCN from an optimization perspective and derive a novel convolutional kernel named GCN+ which has lower parameter amount while relieving the over-smoothing inherently. Extensive experiments on real-world datasets demonstrate the superior performance of GCN+ over state-of-the-art baseline methods on the node classification tasks.
\end{abstract}

\section{Introduction}
Graphs are ubiquitous in the real world, which can easily express various and complex relationships between objectives. In recent years, extensive studies have been conducted on deep learning methods for graph-structured data. There are several approaches on analyzing the graph, including network embedding \cite{perozzi2014deepwalk, tang2015line, grover2016node2vec}, which only uses the graph structure, and graph neural networks (GNNs), which consider graph structure and node features simultaneously. GNNs have shown powerful ability on modeling the graph-structured data in a variety of graph learning tasks such as node classification \cite{2018Large, hamilton2017inductive, yang2016revisiting, kipf2016semi, velivckovic2017graph, wu2019simplifying}, link prediction \cite{zhang2017weisfeiler, zhang2018link, cai2020multi} and graph classification \cite{gilmer2017neural, lee2019self, ma2019graph, xu2018powerful, ying2018hierarchical, zhang2018end}. GNNs have also been applied to a range of applications, including social analysis \cite{qiu2018deepinf, li2019encoding}, recommender systems \cite{ying2018graph, monti2017geometric}, traffic prediction \cite{guo2019attention, li2019predicting}, drug discovery \cite{zitnik2017predicting} and fraud detection \cite{liu2019geniepath}. 

GNNs usually have different design paradigms, which include the spectral graph convolutional networks \cite{bruna2013spectral, defferrard2016convolutional, kipf2016semi},  message passing framework \cite{gilmer2017neural, hamilton2017inductive}, and neighbor aggregation via recurrent neural networks \cite{li2015gated, dai2018learning}. By using the idea of message passing framework, GNNs are to design various graph convolutional layers to update each node representation by aggregating the node representations from its neighbors. 

However, most GNNs only consider the immediate neighbors for each node, which impedes their ability to extract the information of high-order neighbors. More layers usually lead to the performance degradation, which is caused by over-fitting and over-smoothing, of which the former is due to the increasing number of parameters when fitting a limited dataset whereas the latter is the inherent issue of the graph learning. How to make use of the high-order information of neighbors as well as achieving better performance remains a challenge. We need more insights to understand what GCN does and why over-smoothing occurs.
\begin{figure*}[t]
	\centering  
	\subfigure{
		\includegraphics[width=0.12\textwidth]{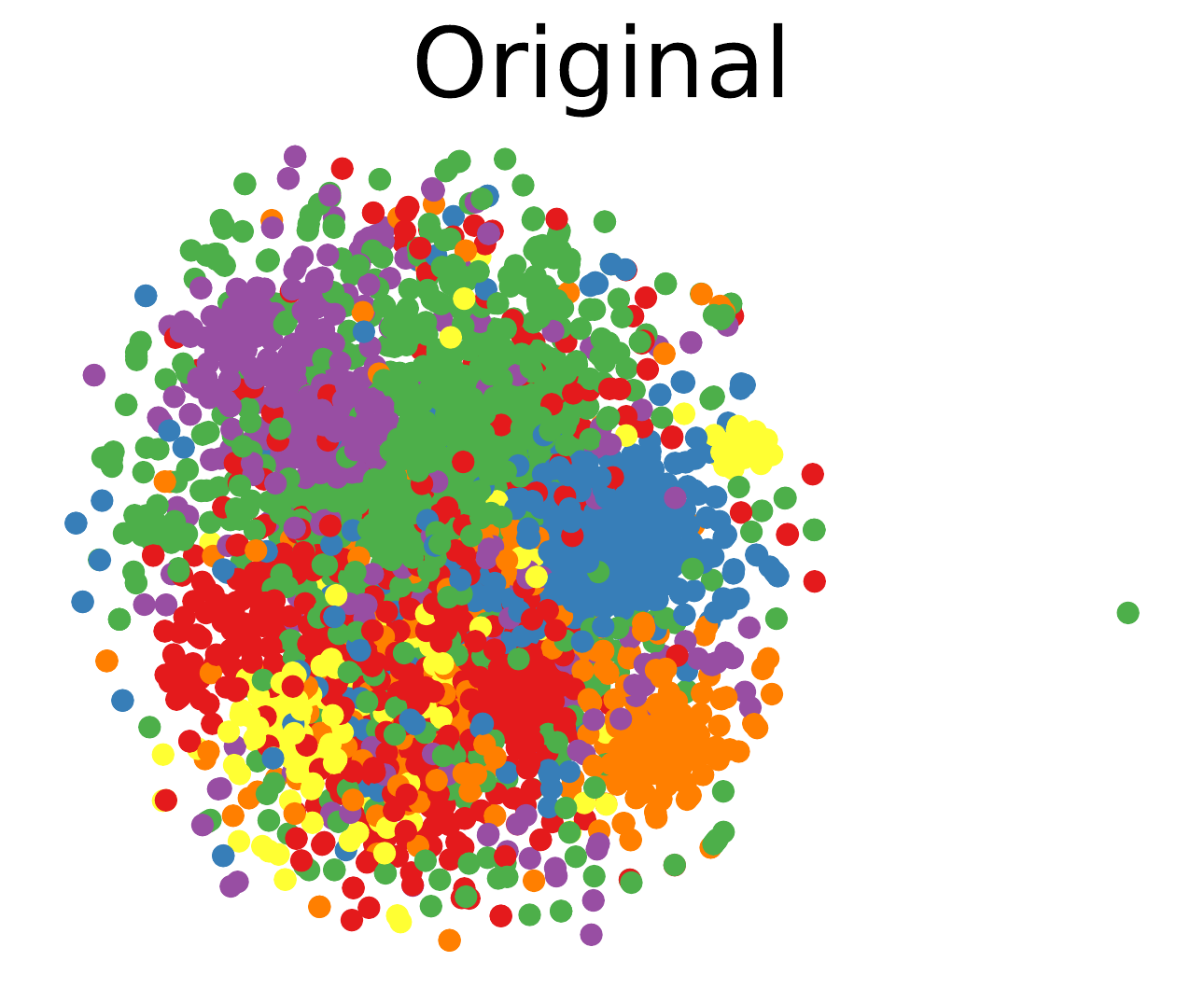}}\hspace{-4mm}
	\subfigure{
		\includegraphics[width=0.12\textwidth]{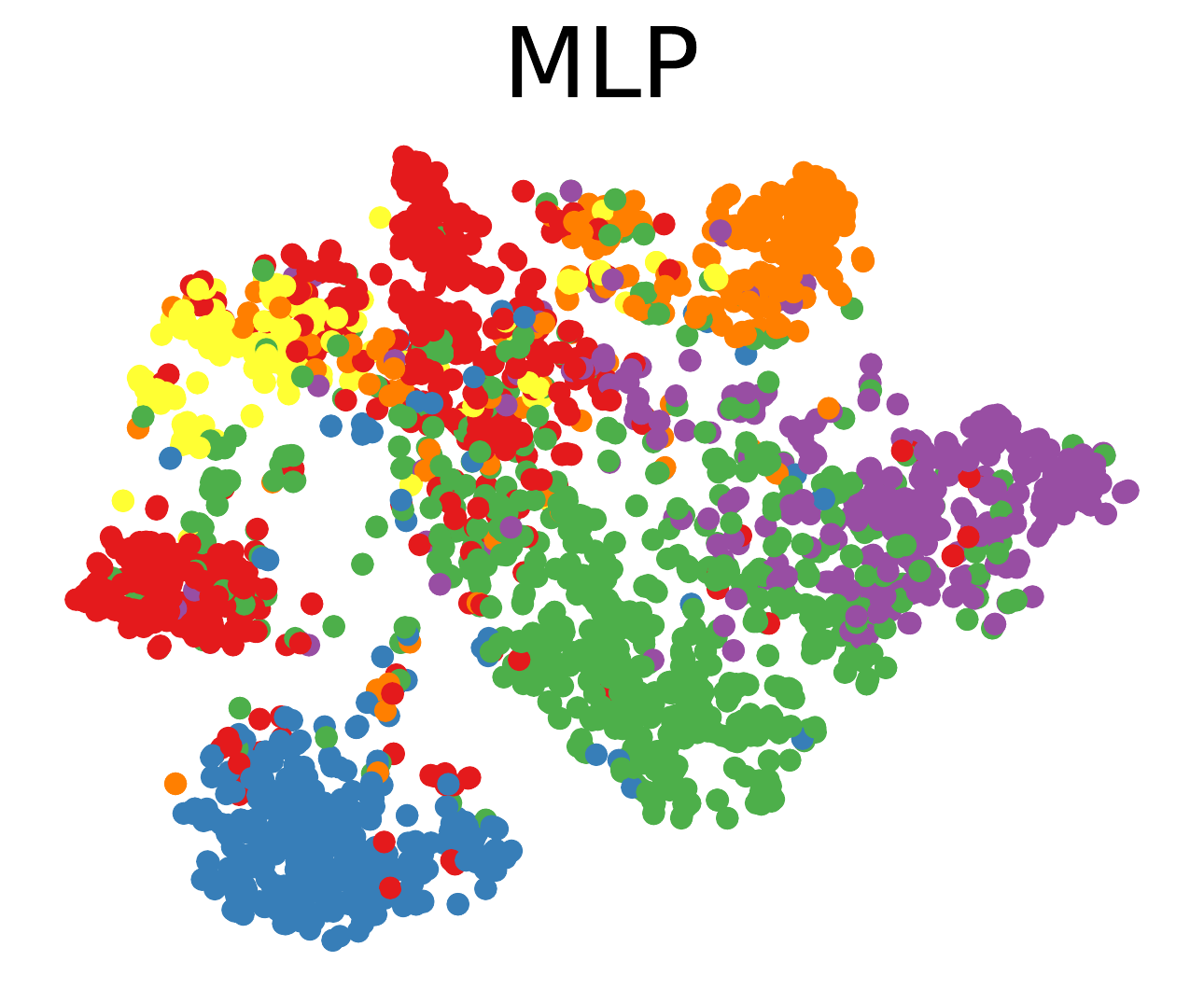}}\hspace{-2mm}
	\subfigure{
		\includegraphics[width=0.12\textwidth]{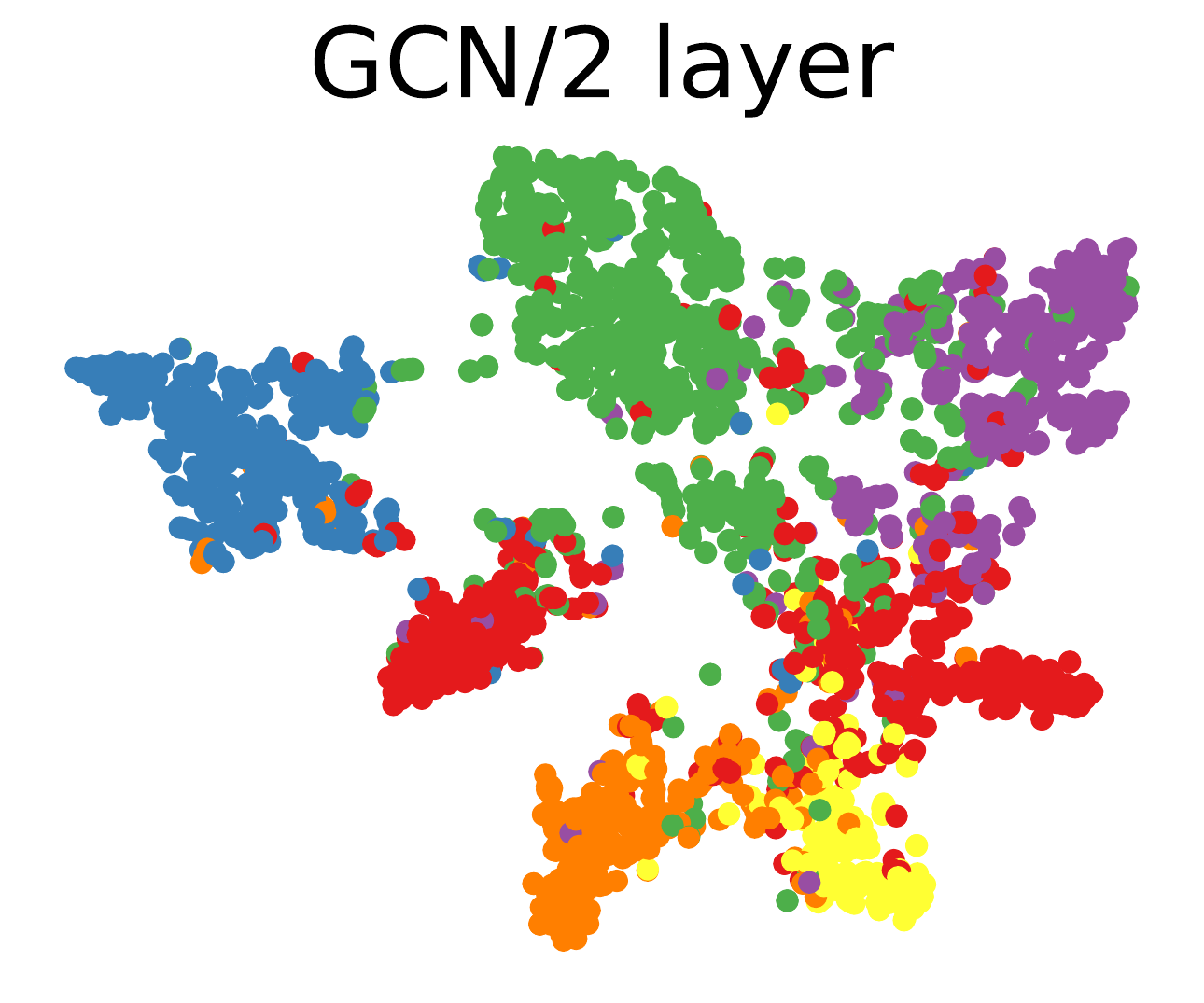}}\hspace{-2mm}
	\subfigure{
		\includegraphics[width=0.12\textwidth]{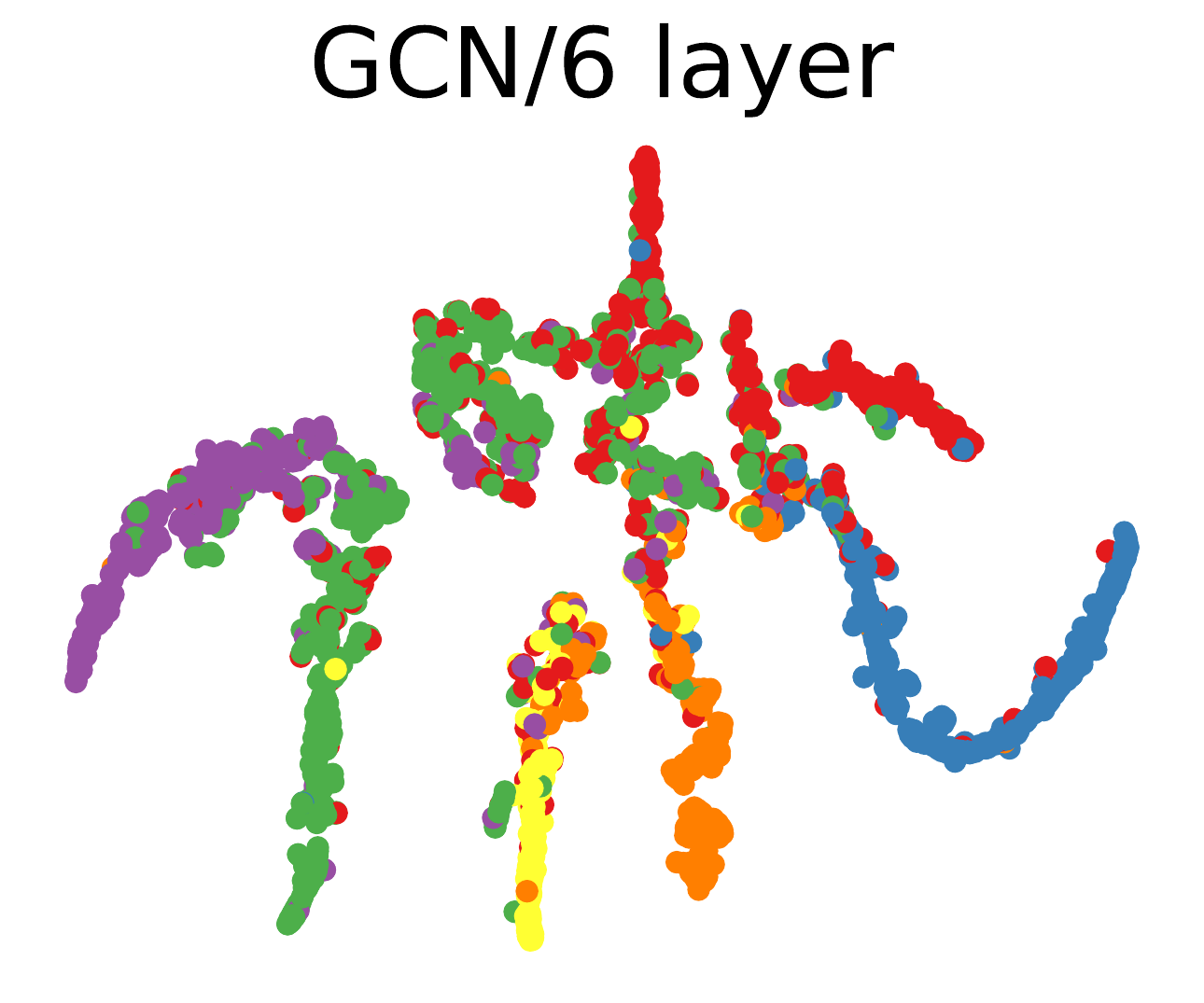}}\hspace{-2mm}
	\subfigure{
		\includegraphics[width=0.12\textwidth]{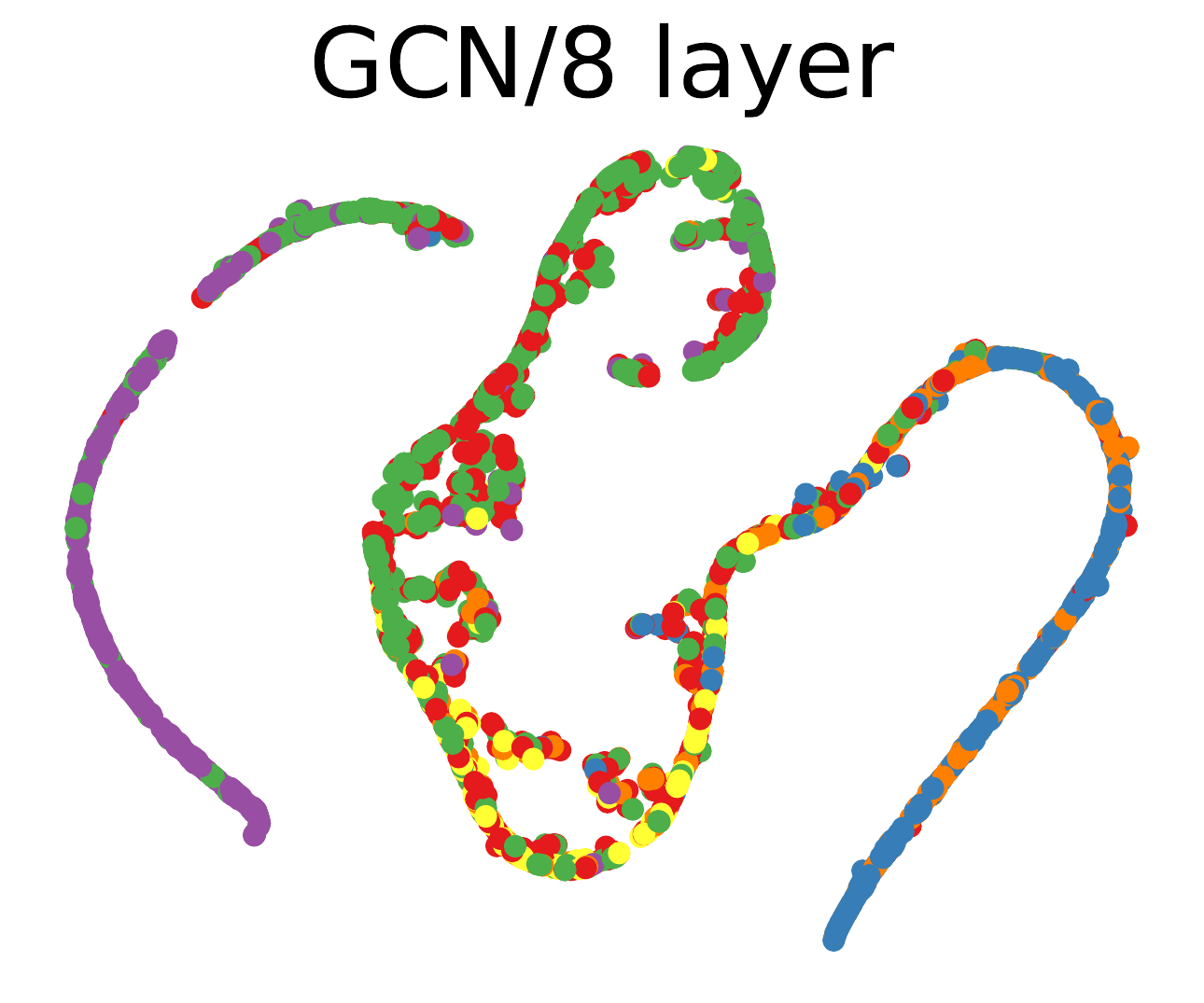}}\hspace{-2mm}
	\subfigure{
		\includegraphics[width=0.12\textwidth]{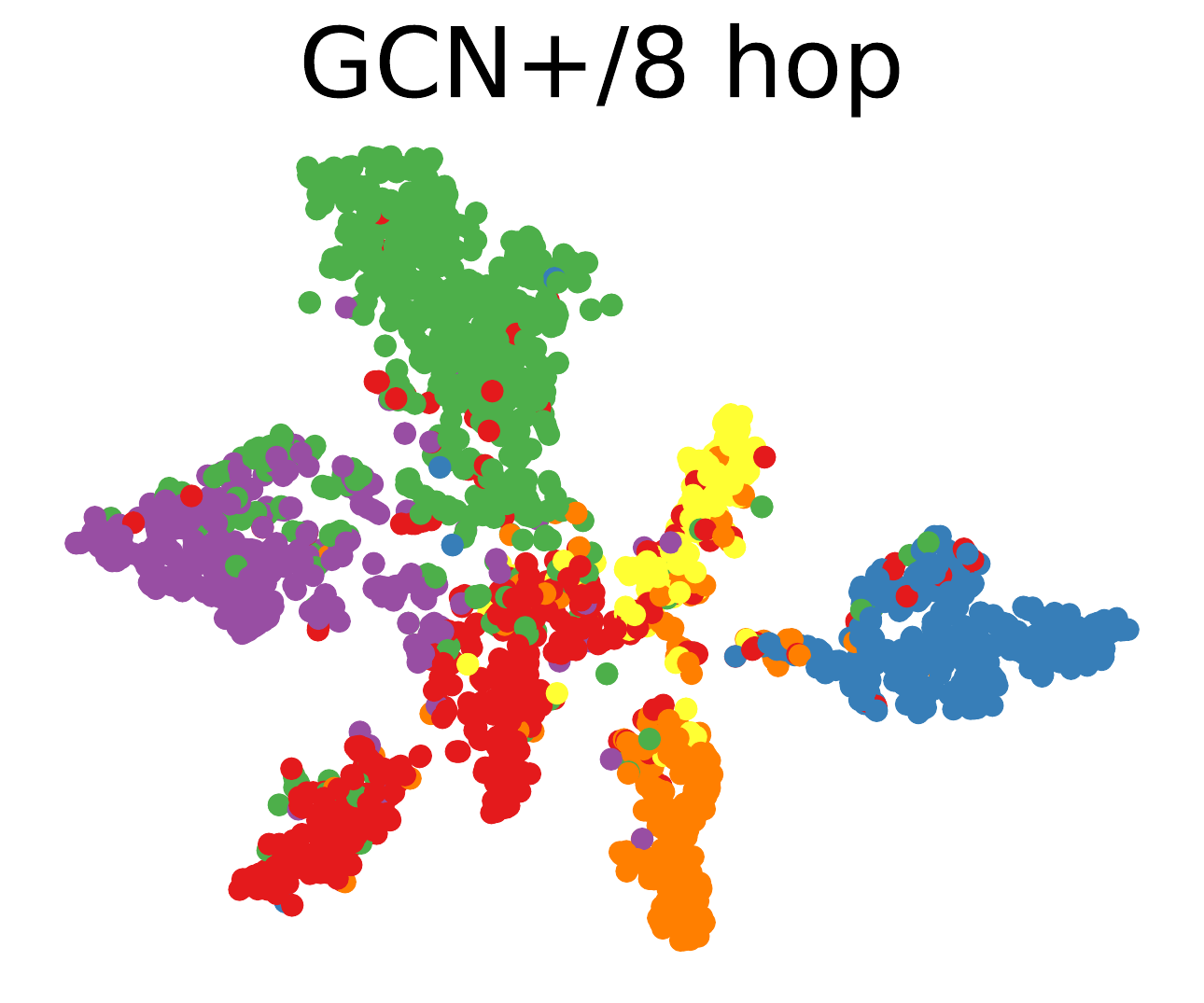}}\hspace{-2mm}
	\subfigure{
		\includegraphics[width=0.12\textwidth]{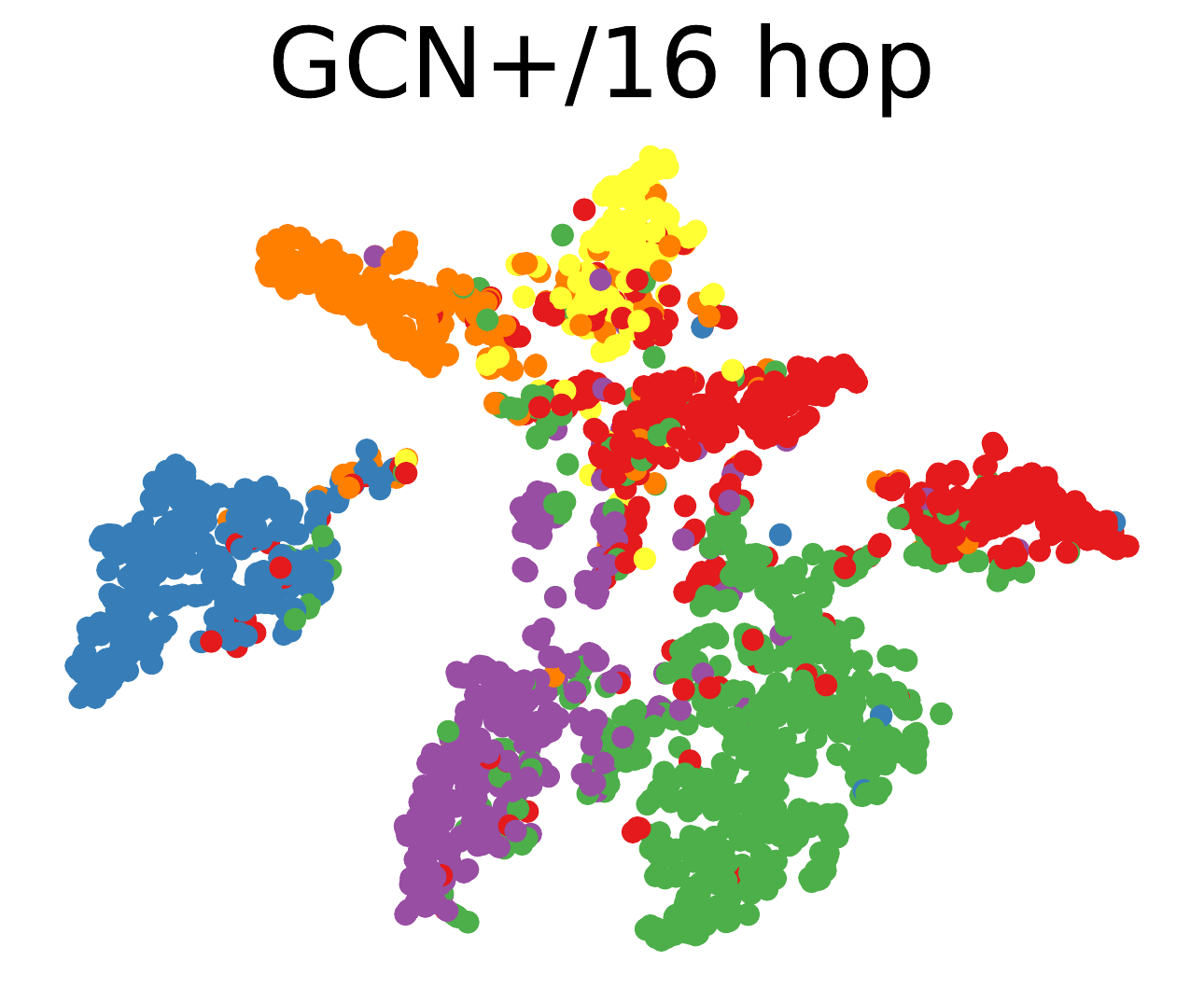}}\hspace{-2mm}
	\subfigure{
		\includegraphics[width=0.12\textwidth]{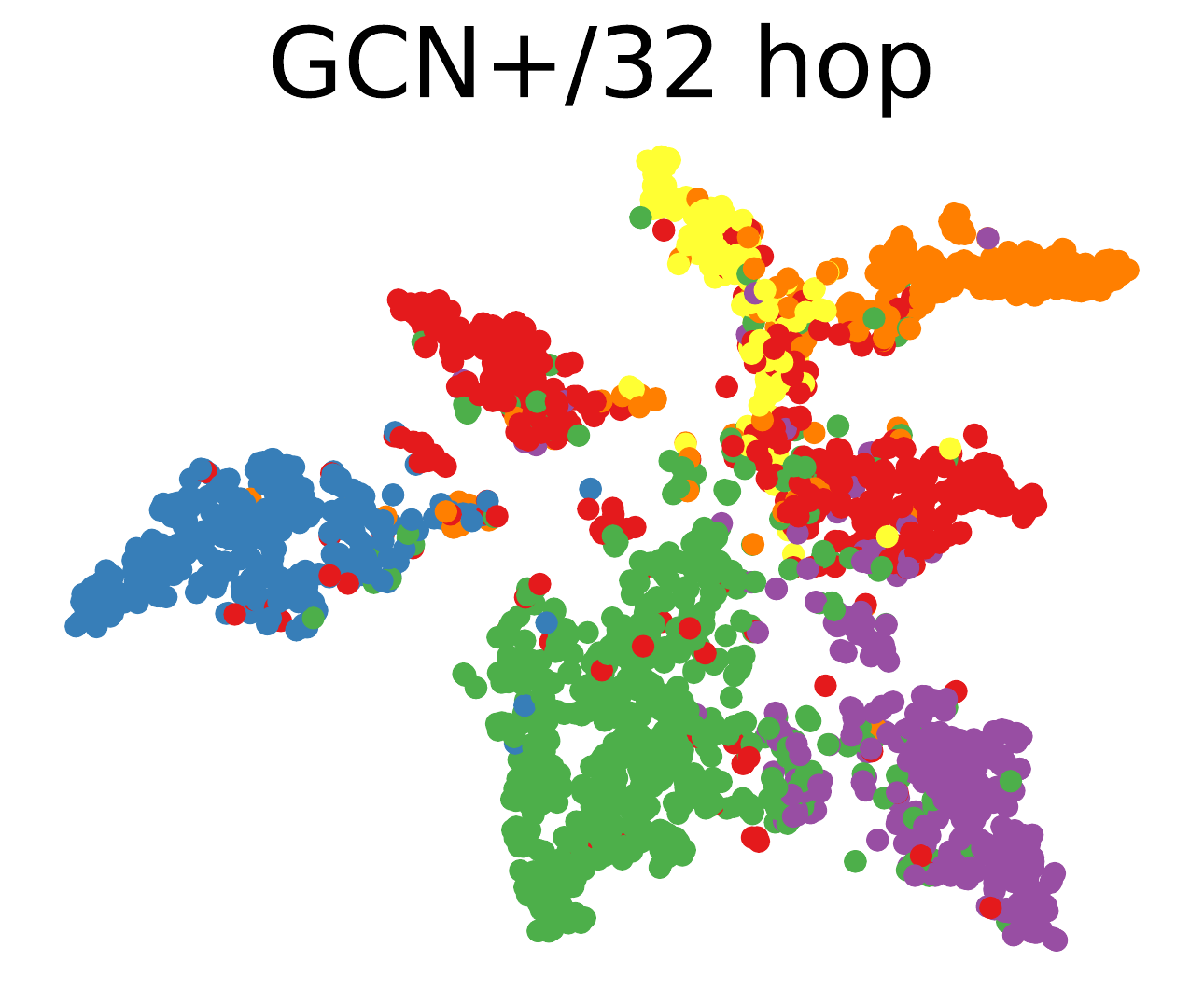}}
	\caption{t-SNE Visualization of learned node representations, which include original features, MLP, different layers of GCN and different hops of GCN+ on \textit{Cora}. Colors represent node classes.}
	\label{model_vs}
\end{figure*}

Several studies \cite{li2018deeper, xu2018representation, klicpera2018predict, chen2020simple, liu2020towards} have noticed over-smoothing, that is after multiple propagations, the final output of vanilla multi-layer GCN converges to a vector which only carries the information of the degree of graph and the node features are indistinguishable. Fig. \ref{model_vs} shows the node representations of vanilla multi-layer GCN on a small citation network \textit{Cora}. We can observe that 2-layer GCN learns a meaningful embeddings which distinguish the different classes whereas more layers degrade the performance and lead to indistinguishable features.

Different from previous studies, we interpret the current graph convolutional operations from an optimization perspective, and argue that over-smoothing is mainly caused by the naive first-order approximation of the solution to the optimization problem. By solving it and applying the first-order approximation, we get the standard GCN kernel. This suggests that the original GCN kernel can be viewed as a simplified version of the solution. We argue that this simplification loses necessary information which is crucial to tackle the over-smoothing to some extent. Based on this observation, two metrics are proposed to measure the smoothness of connected and disconnected pairwise node features respectively. Furthermore, we set three constraints: (a) the embedding learned by GCNs should not be too far off of the original features; (b) the connected nodes should have similar embeddings; (c) the disconnected nodes are assumed to have different embeddings. 

As a result, we build a universal theoretical framework of GCN from an optimization perspective which smooths the node features and regularizes the (disconnected) node feature simultaneously. We consider two different cases of our framework, where the first case contains the current popular GCN \cite{kipf2016semi}, SGC \cite{wu2019simplifying} and PPNP \cite{klicpera2018predict}, and the second case regularizes the pairwise distance of disconnected nodes.

The contributions of this work are summarized as follow:
\begin{itemize}
	\item We provide a universal theoretical framework of GCN from an optimization perspective where the popular GCNs can be viewed as a special case of it. Furthermore, we derive a novel convolutional kernel named GCN+, which relieves the over-smoothing inherently and has lower parameter amount.
	\item We propose two quantitative metric to measure the smoothness and over-smoothness of the final nodes representations, which provides new insight to analyze the over-smoothing.
	\item We conduct extensive experiments on several public real-world datasets. Our results demonstrate the superior performance of GCN+ over state-of-the-art baseline methods.
\end{itemize} 

\section{Notations}
Given an undirected graph $G=(V,E,X)$, $V$ is node set with $|V|=n$, $E$ is edge set. Let $A\in \mathbb{R}^{n \times n}$ denote the adjacency matrix, where $A_{ij}=1$ if there is an edge between node $i$ and node $j$ otherwise 0. Let $D\in \mathbb{R}^{n \times n}$ denote the diagonal degree matrix where $D_{ii}=\sum_{j}A_{ij}$. Each node is associated with $d$ features, and $X \in \mathbb{R}^{n\times d}$ is the feature matrix of nodes. each row of $X$ is a signal defined over nodes. The graph Laplacian matrix is defined as $L=D-A$. Let $\tilde{A}=A+I$ and $\tilde{D}=D+I$ denote the adjacency and degree matrices of the self-loop graph respectively. We denote $\tilde{A}_{\textit{sym}}=\tilde{D}^{-1/2} \tilde{A} \tilde{D}^{-1/2}$ and $\tilde{A}_{\textit{rw}}=\tilde{D}^{-1} \tilde{A}$. Assume that each node $v_i$ is associated with a class label $y_i \in Y$ where $Y$ is a set of $c$ classes. Let $N(v)$ denote the neighbors of $v$ in graph, that is $N(v)=\{u\in V|\{u,v\}\in E\}$ and $\tilde{N}(v)=N(v) \cup \{v\}$.
$L'$ is the Laplacian matrix of the graph $G'(V',E',X)$, which is the complement of $G$, that means $G'$ has the same nodes as $G$ whereas if $\{u,v\}\in E$, then $\{u,v\}\notin E'$. Let $A'$ and $D'$ denote the corresponding adjacency and degree matrix respectively. We have $A'+A=J_n-I_n$ and $D'+D=(n-1)I$ where $J_n$ is a matrix whose element are all 1. Let $\text{num}(E)$ and $\text{num}(E')$ denote the numbers of edges in $G$ and $G'$ respectively, we have $\text{num}(E)+\text{num}(E')=\frac{n(n-1)}{2}$.
\section{Perspectives of GCN}
Here we provide three views to derive or understand the vanilla GCNs.
\subsection{Spectral Graph  Convolution}
\citet{bruna2013spectral} define the spectral convolutions on graph by applying a filter $g_\theta$ in the Fourier domain to a graph signal. ChebNet \cite{defferrard2016convolutional} suggests that the graph convolutional operation can be further approximated by the $k$-th order Chebyshev polynomial of Laplacian. \citet{kipf2016semi} simplify the ChebNet and obtains a reduced version of ChebNet by the renormalization trick:
\begin{equation} \label{gcn}
H^{(l+1)}\!\!=\!\!\sigma(\!\tilde{A}_{\textit{sym}}H^{(l)}W^{(l)})\!\!=\!\!\sigma(\tilde{D}^{-\frac{1}{2}}\tilde{A}\tilde{D}^{-\frac{1}{2}}H^{(l)}W^{(l)}),
\end{equation}
where $\sigma$ denote the activation function such as ReLU. $W^{(l)}$ is a layer-specific trainable weight matrix. $H^{(l)}$ is the feature matrix of $l$-th layer and $H^{(0)}=X$.

\subsection{Message Passing}
Message passing \cite{gilmer2017neural} means that a node on the graph aggregates the message from neighbors and update its embedding:
\begin{equation} \label{mpnn}
h_v^{(l)}=U_l\bigg(h_v^{(l-1)}, \sum_{u \in N(v)}M_l\big(h_u^{(l-1)}, h_v^{(l-1)}, e_{uv}\big)\bigg),
\end{equation}
where $M_l(\cdot)$ and $U_l(\cdot)$ are message aggregation function and vertex update function, respectively. $h_v^{(l)}$ denotes the hidden state of node $v$ at $l$-th layer, and $e_{uv}$ is the edge features.

In this way, GCN layer can be decomposed into two steps, including the neighbors' message aggregation and update:
\begin{equation} \label{gcn-mpnn}
h_v^{(l)}=\sigma\bigg(W^{(l)}\sum_{u\in \tilde N (v)} \frac{h^{(l-1)}_u}{\sqrt{|N(v)||N(u)|}}\bigg).
\end{equation}

Here a GCN layer can be viewed as a weighted average of all neighbors' message where the weighting is proportional to the inverse of the number of neighbors.

\subsection{Graph Regularized  Optimization} \label{GRO}
Let $\bar{X} \in \mathbb{R}^{n\times d}$ denote the final node embeddings matrix, and $\bar{x}_i$ is the $i$-th row of $\bar{X}$. We consider the following optimization problem:
\begin{equation} \label{gcn-opt}
f = \min_{\bar X} \bigg(\sum_{i \in V}\|\bar{x}_i -x_i\|_{\tilde{D}}^2 + \alpha \sum_{\{i,j\} \in E}\|\bar{x}_i -\bar{x}_j\|_2^2\bigg),
\end{equation}
where $(x,y)_{\tilde D}=\sum_{i\in V}d(i)x(i)y(i)$, if $x=y$, we have$\quad \|x\|_{\tilde D}=\sqrt{(x,x)_{\tilde D}}$.

The first term in the above optimization problem is the fitting constraint, which means the output features (also called embeddings) should not be too far off of the input features, while the second term is the smoothness constraint, which means the connected nodes should have similar embeddings. $\alpha > 0$ is a hyperparameter to balance the importance of two objections. It is worth noting that there is no limit to the specific transformation from $X$ to $\bar{X}$.

Before solving the optimization problem, we have the following lemma.
\begin{lemma}
	$\tilde{A}_{\textit{rw}}$ and $\tilde{A}_{\textit{sym}}$ always have the same eigenvalues $|\lambda|\leq1$.
\end{lemma}

\begin{corollary} \label{inverse}
	$(I_n-\alpha\tilde{A}_{\textit{rw}})$ and $(I_n-\alpha\tilde{A}_{\textit{sym}})$ are invertible if $\alpha \in[0,1)$. 
\end{corollary}

\begin{lemma}
	Given a graph with adjacency matrix $A$, the powers of $A$ give the number of walks between any two vertices.
\end{lemma}

\begin{corollary}  \label{high-order}
	$A^k$ includes the information of high-order neighbors.
\end{corollary}
Next, we derive the closed-form solution of Eq. \ref{gcn-opt}. Specifically, we rewrite  Eq. \ref{gcn-opt} as 
$$
f=\min_{\bar X} \bigg(\text{Tr}\big((\bar X-X)(\bar X-X)^T\tilde{D}\big) + \alpha\text{Tr}({\bar X}^TL \bar X)\bigg).
$$
Differentiating $f$ with respect to $\bar{X}$, we have 
$$
\frac{df}{d\bar{X}}=\tilde{D}(\bar{X}-X)+\alpha L\bar{X}=0.
$$
Notice Corollary \ref{inverse}, we have
$$
\bar{X}=(1-\mu) (I-\mu\tilde {A}_{\textit{rw}})^{-1}X,
$$
where $\mu=\frac{\alpha}{1+\alpha}$.

Actually, the solution is also the personalized PageRank \cite{page1999pagerank}' s limiting distribution. If we set $\mu=0.5$, we get $\bar X=(2I-\tilde {A}_{\textit{rw}})^{-1}X$, 
and $\tilde {A}_{rw}X$ is the first-order Taylor approximation. 
By replacing $\tilde {A}_{rw}$ with $\tilde {A}_{\textit{sym}}$, we get standard graph convolution kernel. In other words, we lose the information from high-order neighbors, which is contained in the error series of the Taylor expansion. (See Corollary \ref{high-order}).

In a nutshell, we obtains the well-known kernel or resemble form of the graph convolution from different ways. 
\section{Over-smoothing in Vanilla Deep GCN}

Neural network usually performs better when stack more layers while graph neural network does not benefit from the depth. On the contrary, more layers often result in significant degradation in performance. 

Previous work illustrates the over-smoothing by computing the limiting distribution of $A_{k}$ when $k \rightarrow \infty$, Actually, this is not identical with vanilla deep GCN, which contains non-linear transformation among different layers. Litter work considers the non-linearity in multi-layer GCN. \citet{oono2019graph} extend the linear analysis to the non-linearity firstly, which considers the ReLU activation function. They suggest that the node features of a $k$-layer GCNs will converge to a subspace and incur information loss, which makes the node feature indistinguishable.

At first, one main reason we introduce the deep architecture in GCN is that we want to use the long-range neighbor's information. We argue that vanilla deep GCN is not the correct way to capture this information. However, It does not mean that deep architecture is useless. \citet{chen2020simple} and \citet{liu2020towards} have shown that more layers can boost the performance of GCN on several datasets and tasks.

To quantify the over-smoothing in vanilla deep GCN, we compute the overall pairwise distance of node embeddings as follows:
\begin{equation} 
\nonumber
\begin{aligned}
\!M_{\textit{\!overall}}&\!=\!\!\!\sum_{i,j\in\!V}\!\!\|\bar{x}_i\!-\!\bar{x}_j\|_2^2\!=\!\!\!\!\sum_{\{i,j\}\in\!E}\!\!\! \|\bar{x}_i\!-\!\bar{x}_j\|_2^2\!\!+\!\!\!\!\sum_{\{i,j\}\notin\!E}\!\!\|\bar{x}_i\!-\!\bar{x}_j\|_2^2 \\
&= \text{Tr}({\bar X}^TL \bar X)+ \text{Tr}({\bar X}^TL' \bar X).\\
\end{aligned}
\end{equation}
\begin{figure}[t]
	\subfigure[Overall distance]{
		\label{overall}
		\includegraphics[width=0.22\textwidth]{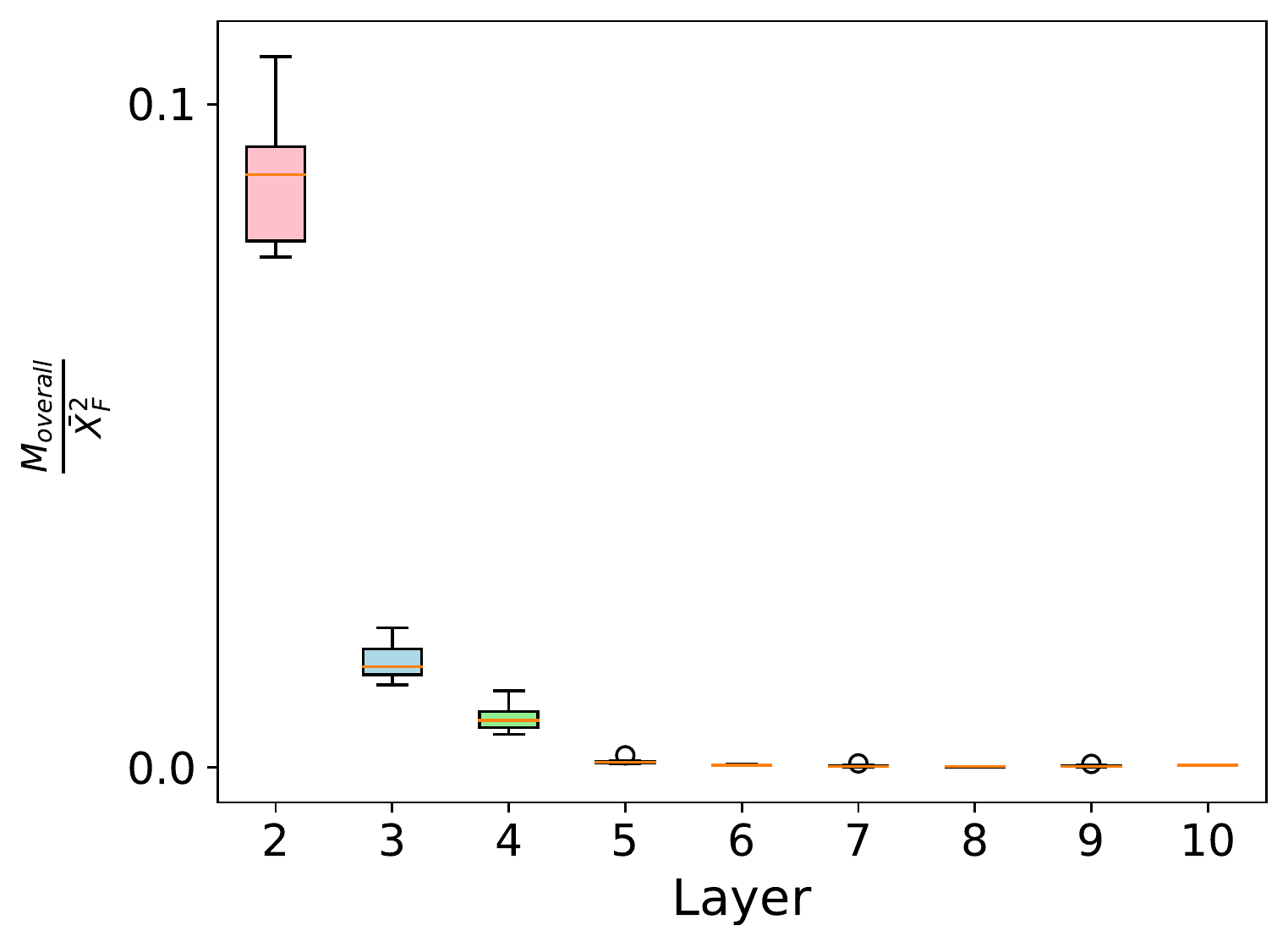}}\hspace{-1mm}
	\subfigure[Fraction of diatsnce]{
		\label{fraction}
		\includegraphics[width=0.24\textwidth]{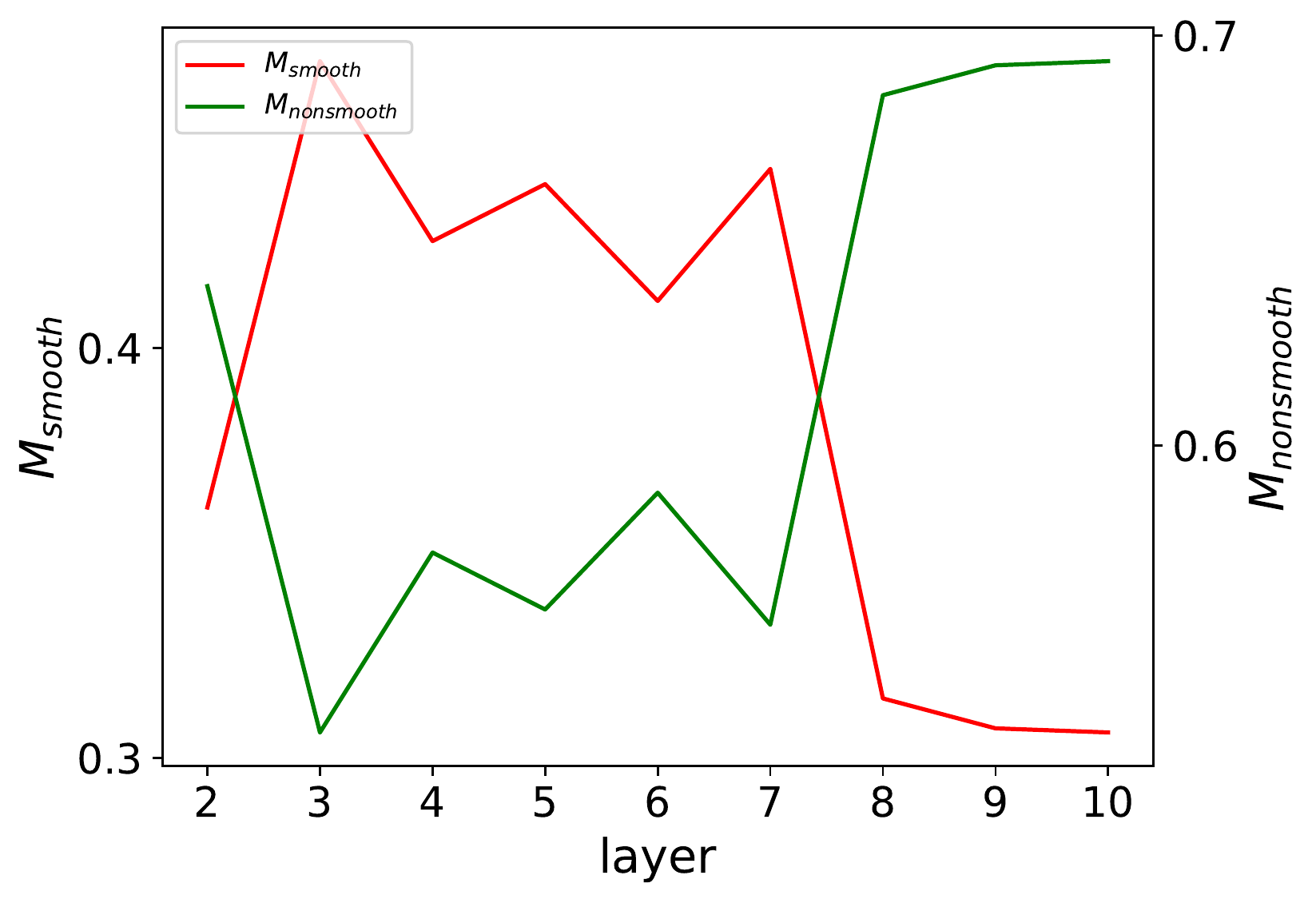}}
	\caption{ $M_{\textit{\!overall}}, M_{\textit{\!smooth}}$ and $M_{\textit{\!non-smooth}}$ of the output node embeddings of Vanilla GCN with increasing layers on \textit{Cora}.}
	\label{boxplot}
\end{figure}

Fig. \ref{overall} depicts the pairwise distance distribution of vanilla GCN with increasing layers on \textit{Cora}. We can see that $M_{overall}$ decreases as the model goes deeper. Revisit the two parts of $M_{overall}$, we propose two fine quantitative metrics to measure the over-smoothing of graph representation.
\begin{equation} 
\begin{aligned}
M_{\textit{smooth}} &= D_{\textit{smooth}}/D_{\textit{overall}},\\
M_{\textit{non-smooth}} &= D_{\textit{non-smooth}} /D_{\textit{overall}},\\
\end{aligned}
\end{equation}
where
\begin{equation} 
\begin{aligned}
D_{\textit{smooth}} &= \text{Tr}({\bar X}^TL \bar X)/\text{num}(E) ,\\
D_{\textit{non-smooth}} &= \text{Tr}({\bar X}^TL' \bar X)/\text{num}(E'),\\
D_{\textit{overall}}&=D_{\textit{smooth}}+D_{\textit{non-smooth}}.
\end{aligned}
\end{equation}

Here, $\text{num}(E)$ and $\text{num}(E')$ in the denominator are used to eliminate the impact of unbalanced edge numbers in $G$ and $G'$. $M_{\textit{smooth}}$ measures the smoothness of the graph representation of connected pair nodes while $M_{\textit{non-smooth}}$ measures the smoothness of the graph representation of disconnected pair nodes. 

Fig. \ref{fraction} compares $M_{\textit{smooth}}$ and $M_{\textit{non-smooth}}$. We see that $M_{\textit{smooth}}$ contributes to quite a few parts of the overall distance, which seems counter-intuitive. We will discuss this two metrics of GCN+ in Section \ref{oversmooth-of-gcn+} again.

\section{A General Framework of GCN}
Recall the graph regularized optimization problem, we add a negative term to constrain the sum of distances between disconnected pairs as follow:
\begin{equation} \label{full-opt}
\nonumber
	f \!\!=\!\min_{\bar X}\!\!\bigg(\!\sum_{i \in V}\|\bar{x}_i \!-\!x_i\|_{\tilde{D}}^2 \!+\! \alpha\!\!\!\!\! \sum_{\{i,j\} \in E} \! \|\bar{x}_i \!-\!\bar{x}_j\|_2^2 \!-\!\beta\!\!\!\!\! \sum_{\{i,j\} \notin E} \!\|\bar{x}_i \!-\!\bar{x}_j\|_2^2\!\bigg)\!,
\end{equation}
where $\alpha$ and $\beta$ are hyperparameters to balance the importance of the corresponding terms.

%

We consider two cases: $\beta=0$ and $\beta \neq0$.
\subsection{Case 1:$\beta=0$}
In this situation, $\bar{X}=(1-\mu)(I_n-\mu\tilde{A}_{\textit{rw}})^{-1}X$ where $\mu=\frac{\alpha}{1+\alpha}\in(0,1)$. Directly calculating such an intractable expression is not only computationally inefficient but also results in a dense $\mathbb{R}^{n \times n}$matrix. It would lead to a high computational complexity and memory requirement when we apply such operator on large graphs. We can achieve linear computational complexity via power iteration.

We use $\tilde{A}$ to denote $\tilde{A}_{\textit{sym}}$ and $\tilde{A}_{\textit{rw}}$. Here we consider a more general expression $(1-\mu)(I_n-\mu \tilde{A})^{-1}H$ where $H=H=f_{\theta}(X)$.
\begin{theorem} \label{theorem1}
	$(I_n-\mu\tilde{A})$ is invertible. Consider the following iterative scheme
	\begin{equation} \label{iterative1}
	\begin{aligned}
	Z^{(0)}&=H, \\
	Z^{(k)}&=\mu\tilde{A}Z^{(k-1)}+(1-\mu) H,
	\end{aligned}
	\end{equation} where $\mu \in (0,1)$.
	When $k \rightarrow \infty$, 
	\begin{equation}
Z^{(\infty)}=(1-\mu)(I_n-\mu\tilde{A})^{-1}H.
	\end{equation} 
\end{theorem}

\begin{proof}
	Using corollary \ref{inverse}, we can see that $(I_n-\mu\tilde{A})$ is invertible. Combining the two equation of \ref{iterative1}, we have 
	$$
	Z^{(k)}=\bigg(\mu^k \tilde{A}^k+(1-\mu) \sum_{i=0}^{k-1} \mu^i \tilde{A}^i\bigg)H.
	$$
	Notice that
	\begin{equation} 
	\begin{aligned}
	\lim_{k \rightarrow \infty}\mu^k \tilde{A}^k&=0, \\
	\lim_{k \rightarrow \infty}\sum_{i=0}^{k-1} \mu^i \tilde{A}^i&=\big(I-\mu\tilde{A}\big)^{-1}.
	\end{aligned}
	\end{equation}
	Hence, the proof is finished. 

\end{proof}

Actually, the prevalent GCN, SGC and APPNP can be viewed as the special variant of Case 1.
\subsection{Case 2:$\beta \neq 0$} 

In this situation, $\bar{X}=Q^{-1}$ if $Q=\big((1+\alpha+\beta)I-(\alpha+\beta)\tilde D^{-1}\tilde A-\beta n\tilde D^{-1}+\beta \tilde D^{-1} \mathbf J_n\big)$ is invertible when  we choose a suitable $\beta$. We will introduce the conditions later.


First we use the  first-order Taylor approximation of above convolutional kernel ($\text{GCN}^*$) directly without any tricks such as Batch Normalization \cite{ioffe2015batch} or residual connection \cite{he2016deep} on two small citation datasets \textit{Cora} and \textit{Citeseer}. We compare the performance of the vanilla deep GCN and $\text{GCN}^*$ as the model layer increases. Fig. \ref{figure1} shows the result of GCN and $\text{GCN}^*$.  Dashed lines illustrate the performance of GCN, which shows that deep GCN suffers from performance drop. We can see that the performance decay with $\text{GCN}^*$ kernel is much slower.

\citet{oono2019graph} have proved that the node feature of vanilla $k$-layer GCN will converges to an invariant subspace which only carry the information of the connected component and node degree. The convergence speed is proportional to the $\lambda^k$, where $\lambda$ is the supremum of eigenvalue of $\tilde{A}$. In GCN*, $\lambda>1$(see the proof of Theorem \ref{theorem2}), which implies that $\lambda^k$ is large, thus the information loss and over-smoothing are relieved.

Although the modified graph kernel relieves over-smoothing to some extent, more layers do not boost the performance, which is not our focus. However the above result demonstrates that it is an efficient way to tackle the over-smoothing issue. We can achieves linear computational complexity via power iteration similar to Case 1.

\begin{figure}[t]
	\centering
	\includegraphics[width=0.9\columnwidth]{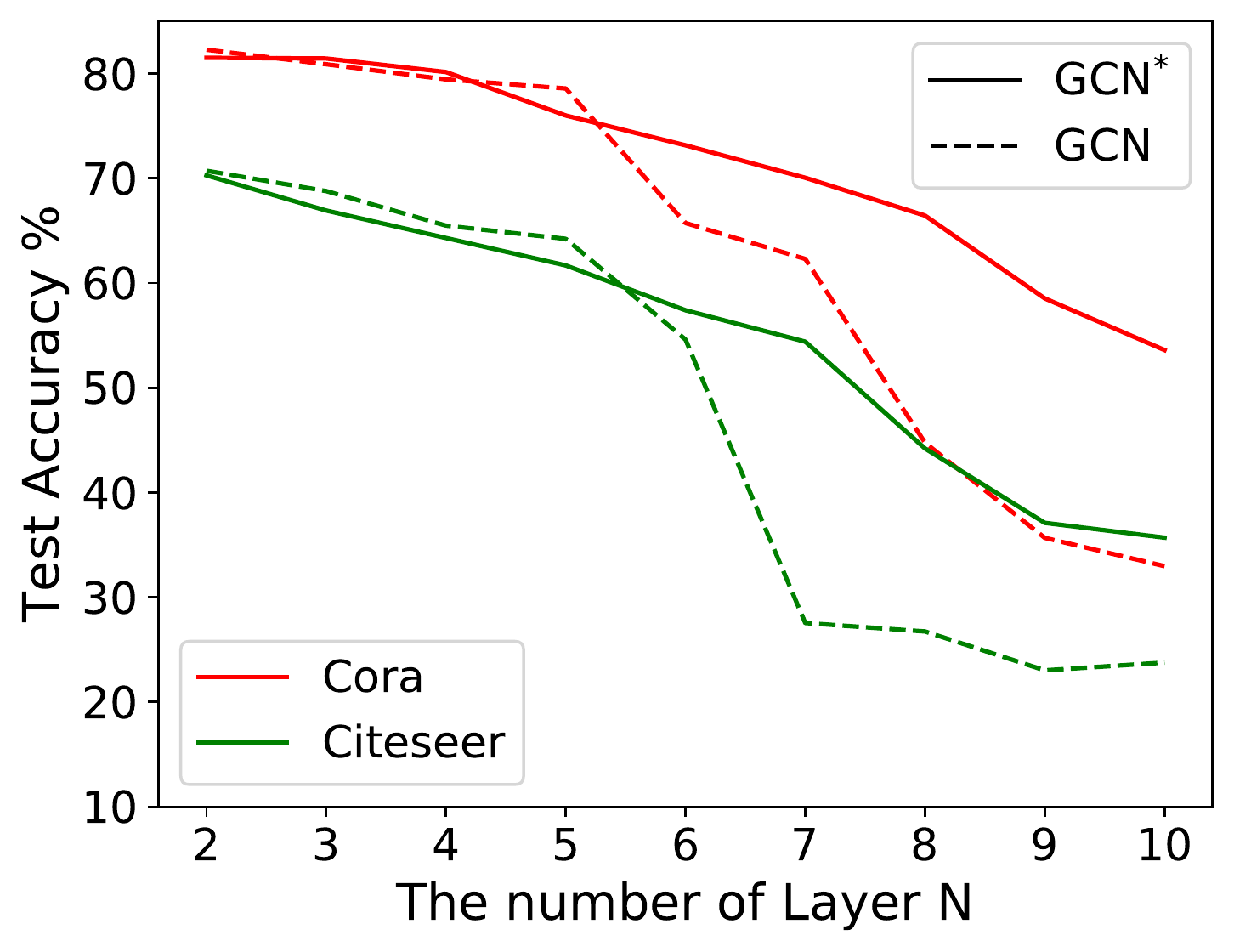} 
	\caption{Performance comparison of vanilla deep GCN vs. $\text{GCN}^*$ with increasing layers on two small datasets.}
	\label{figure1}
\end{figure}

\begin{theorem} \label{theorem2}
	$(I_n-\mu \hat{A})$ is invertible when $\beta<\frac{1}{n}$ where $\mu=\frac{\alpha+\beta}{1+\alpha+\beta}$ and $\hat{A}=\tilde{A}+\frac{\beta n\tilde D^{-1}-\beta \tilde D^{-1} \mathbf J_n}{\alpha+\beta}$.  Consider the following iterative scheme
	\begin{equation} \label{iterative2}
	\begin{aligned}
	Z^{(0)}&=H, \\
	Z^{(k)}&=\mu\hat{A}Z^{(k-1)}+ (1-\mu)H,
	\end{aligned}
	\end{equation} where $\mu \in (0,1)$.
	When $k \rightarrow \infty$, 
	\begin{equation}
	Z^{(\infty)}=(1-\mu)(I_n-\mu\hat{A})^{-1}H.
	\end{equation} 
\end{theorem}
\begin{proof}
	Let $M=\frac{\beta n\tilde D^{-1}-\beta \tilde D^{-1} \mathbf J_n}{\alpha+\beta}=\frac{\beta \tilde D^{-1} }{\alpha+\beta}(nI-J_n)$. Note that $(nI-J_n)$ has the largest eigenvalue $n$. Suppose that $\lambda$ is the eigenvalue of $M$, we have $\lambda \leq \frac{\beta n}{\alpha+\beta} $. Then eigenvalue of $ \hat{A}$ is less than $1+\frac{\beta n}{\alpha+\beta}$. 
	 $(I_n-\mu \hat{A})$ is invertible iff $\frac{1}{\mu}$ is not an eigenvalue of $ \hat{A}$. Note that $\frac{1}{\mu}=\frac{1+\alpha+\beta}{\alpha+\beta}=1+\frac{1}{\alpha+\beta}$, when $\beta<\frac{1}{n}$ we have $\frac{1}{\mu}>1+\frac{\beta n}{\alpha+\beta}$, hence $\frac{1}{\mu}$ cannot be an eigenvalue of $\hat{A}$ and $(I_n-\mu \hat{A})$ is invertible. The proof of the iterative scheme follows the similar procedure of case 1 with a slight difference, as it is trivial, we omit the proof.
\end{proof}

\subsection{Why GCN+ relieve the over-smoothing?}
We have no assumptions on the specific transformation from $X$ to $\bar{X}$. In our implementation, the mathematical expression of GCN+ is defined as 
\begin{equation} \label{gcn+}
\begin{aligned}
Z^{(0)}&=H=\sigma(XW_1), \\
Z^{(k)}&=\mu\hat{A}Z^{(k-1)}+ (1-\mu)H,\\
X_{out}&=\text{softmax}(Z^{(k)}W_2),\\
\end{aligned}
\end{equation} 
where $W_1\in \mathbb{R}^{d \times m}$ and $W_2\in \mathbb{R}^{m \times c}$ are learnable weight matrices, $k$ is the dimension of the hidden layers. 


We interpret the anti-oversmoothing of GCN+ from two ways. First, note that in the power iterative scheme, a fraction of initial node features $H$ is always preserved in each iteration, which can be viewed as a flexible version of residual connection. In addition, we can also understand GCN+ from the frequency of graph signal. In Section \ref{GRO} , we have shown that the original GCN is corresponding to the first-order Taylor approximation of the optimization solution, that means we lost the high frequency part of the signal which contains the high-order information. Actually we omit the error series when we approximate the GCN.


Recall the current representative methods: DAGNN and JKNet, which shows promising improvement than the original GCN. The core formulas of them are as follows:
\begin{equation} 
\nonumber
\begin{aligned}
\text{DAGNN:}&\\
&Z=\text{stack}(H, Z^{(1)}, ..., Z^{(k)}), \quad Z^{(k)}=\hat{A}^{(k)}H,\\
\text{JKNet:\quad}&\\
&Z=\text{Aggr}(Z^{(1)}, ...,Z^{(k)}), \\
\end{aligned}
\end{equation} 
where \text{Aggr} includes \textit{Concatenation}, \textit{Max-pooling} and \textit{LSTM-attention}.

Actually, DAGNN and JKNet both make use of the information which from the immediate and high-order neighbors while GCN+ also benefit from this. Moreover, we provide the theoretical and empirical evidence of GCN+.

\subsection{Parameters Amount}
It is worth noting that the power iterative schemes are parameter-free in two versions of GCN+,  which is similar to APPNP \cite{klicpera2018predict}. In particular,  GCN+ ($\beta=0$) adopts the same scheme as APPNP, where we re-implement it and achieve more impressive results. 
\section{Experiments}
In this section, we evaluate the performance of GCN+ on several benchmark datasets against various graph neural networks on semi-supersized node classification tasks.
	
\subsection{Experimental Setup}
\subsubsection{Datasets}
We conduct extensive experiments on the node-level tasks on two kinds commonly used networks: Planetoid: \textit{Cora}, \textit{CiteSeer}, \textit{Pubmed} \cite{sen2008collective} and recent \textit{Open Graph Benchmark} (OGB) \cite{hu2020open}:\textit{ogb-arxiv}, \textit{ogb-proteins}. The statistics of datasets are summarized in Table \ref{dataset}.  It is worth nothing that OGB includes enormous challenging and large-scale datasets than Planetoid. We refer readers to \cite{hu2020open} for more details about OGB datasets.
\begin{table} 
	\small
	\setlength{\tabcolsep}{1mm}{
	\begin{tabular}{{cccccc}}
		\toprule
		Dataset & Nodes  & Edges    & Classes & Features  &Metric\\ 
		\midrule
		Cora  & 2708     & 5429      &  7   &  1433 & Accuracy \\ 
		Citeseer  & 3327     & 4732      &  7   &  2703  & Accuracy \\ 
		Pubmed  & 19717     & 44338      &  3   &  500  & Accuracy \\ 
		ogb-arxiv  & 169343     & 1166243      &  40   &  128  & Accuracy \\ 
		ogb-proteins & 132534     & 39561252      &  112   &  8 & ROC-AUC  \\ 
		\bottomrule
	\end{tabular}}
\caption{Dataset statistics.} 
\label{dataset}
\end{table}
\subsubsection{Implementations}
We choose the optimizer and hyperparameters of GNN models as follows. We use the Adam optimizer \cite{kingma2014adam} to train all the GNN models with a maximum of 1500 epochs. We set the number of hidden units to 64 on \textit{Cora}, \textit{Citeseer} and \textit{Pubmed} , 256 on \textit{ogb-arxiv} and \textit{ogb-proteins}. For SGC, we vary number of layer in \{1, 2, ..., 10, 15, ..., 60\} and for GCN and GAT in \{2, 4, ..., 10, 15, ..., 30\}. For $\alpha$ in APPNP, we search it from \{0.1, 0.2, 0.3, 0.4, 0.5\}. For DAGNN and JKNet, we search layers from \{2, 3, ..., 10\}. For learning rate, we choose from \{0.001, 0.005, 0.01\}. For dropout rate, we choose from \{0.1, 0.2, 0.3, 0.4, 0.5\}. We perform a grid search to tune the hyperparameters for other models based on the accuracy on the validation set. We run each experiment 10 times and report the average.  

In practice, we use Pytorch \cite{paszke2019pytorch} and Pytorch Geometric \cite{Fey/Lenssen/2019} for an efficient GPU-based implementation of GCN+.
All experiments in this study are conducted on NVIDIA TITAN RTX 24GB GPU.

\begin{table*}[t]
	\small
	\centering
	\setlength{\tabcolsep}{1.5mm}{
	\begin{tabular}{ccccccc}
		\toprule
		\multirow{2}{*}{model} & \multicolumn{2}{c}{\textit{Cora}} & \multicolumn{2}{c}{\textit{Citeseer}} & \multicolumn{2}{c}{\textit{Pubmed}}                             \\
		& Fixed      & Random      & Fixed        & Random        & Fixed        & Random \\
		\midrule
		MLP & $61.6\pm0.6$   & $59.8\pm2.4$& $61.0\pm1.0$  & $58.8\pm2.2$  & $74.2\pm0.7$   &  $70.1\pm2.4$    \\
		GCN\cite{kipf2016semi} & $81.3\pm0.8$   & $79.1\pm1.8$& $71.1\pm0.7$  & $68.2\pm1.6$  & $78.8\pm0.6$   &  $77.1\pm2.7$    \\
		 GAT\cite{velivckovic2017graph}  &$83.1\pm0.4$ &$80.8\pm1.6$ &$70.8\pm0.5$ &$68.9\pm1.7$ & $79.1\pm0.4$  &$77.8\pm2.1$   \\
		 SGC\cite{wu2019simplifying} & $81.1\pm0.5$   & $80.4\pm0.3$& $71.9\pm0.3$  & $71.8\pm0.3$  & $78.9\pm0.0$   &  $77.8\pm0.6$    \\
		 JKNet\cite{xu2018representation}& $80.7\pm0.9$ &$79.2\pm0.9$ &$70.1\pm0.6$&$68.3\pm1.8$ & $78.1\pm0.6$ & $77.9\pm0.9$    \\
		 APPNP\cite{klicpera2018predict} &$83.3\pm0.5$ &$81.9\pm1.4$ &$71.8\pm0.4$ &$69.8\pm1.7$ & $80.1\pm0.2$  &$79.5\pm2.2$     \\
		 DAGNN\cite{liu2020towards}  &$84.4\pm0.5$ &$\textbf{83.7}\pm\textbf{1.4}$ &$73.3\pm0.6$ &$71.2\pm1.4$ & $\textbf{80.5}\pm\textbf{0.5}$  &$80.1\pm1.7$     \\
		\midrule
		 GCN+($\beta=0$)  &$85.2\pm0.5$ &$83.3\pm1.1$ &$73.3\pm0.5$ &$72.3\pm0.7$ & $80.4\pm0.6$  &$80.1\pm0.6$     \\
		 GCN+($\beta \ne 0$)   &$\textbf{85.6}\pm\textbf{0.4}$ &$83.6\pm1.3$ &$\textbf{73.5}\pm\textbf{0.4}$ &$\textbf{72.5}\pm\textbf{0.9}$ & $80.5\pm0.6$  &$\textbf{80.3}\pm\textbf{0.7}$     \\
		\bottomrule  
	\end{tabular}}
\caption{Summary of classification accuracy(\%) on Planetoid datasets of semi-supervised node classification.}
\label{Plantoid-semi}
\end{table*}

\begin{table}[t]
	\centering
	\setlength{\tabcolsep}{3.5mm}{
	\begin{tabular}{{ccccc}}
		\toprule
		Dataset & \textit{ogb-arxiv}  & \textit{ogb-proteins}  \\ 
		\midrule
		GCN  &  $71.74\pm0.29$           &  $72.51\pm0.35$     \\ 
		GraphSAGE  &  $71.49\pm0.25$          &  $77.68\pm0.20$      \\ 
		\midrule
		GCN+($\beta=0$)  & $71.85\pm0.23$&$78.63\pm0.28$  \\ 
		GCN+($\beta \ne0$)  &$\textbf{71.95}\pm\textbf{0.28}$ & $\textbf{79.07}\pm\textbf{0.34}$  \\ 
		\bottomrule
	\end{tabular}}
	\caption{Summary of classification performance(\%) on OGB datasets. For  \textit{ogb-arxiv}, it indicates accuracy and for \textit{ogb-proteins}, it indicates ROC-AUC.} 
	\label{ogb-semi}
\end{table}

\begin{figure}
	\centering
	\includegraphics[width=0.9\columnwidth]{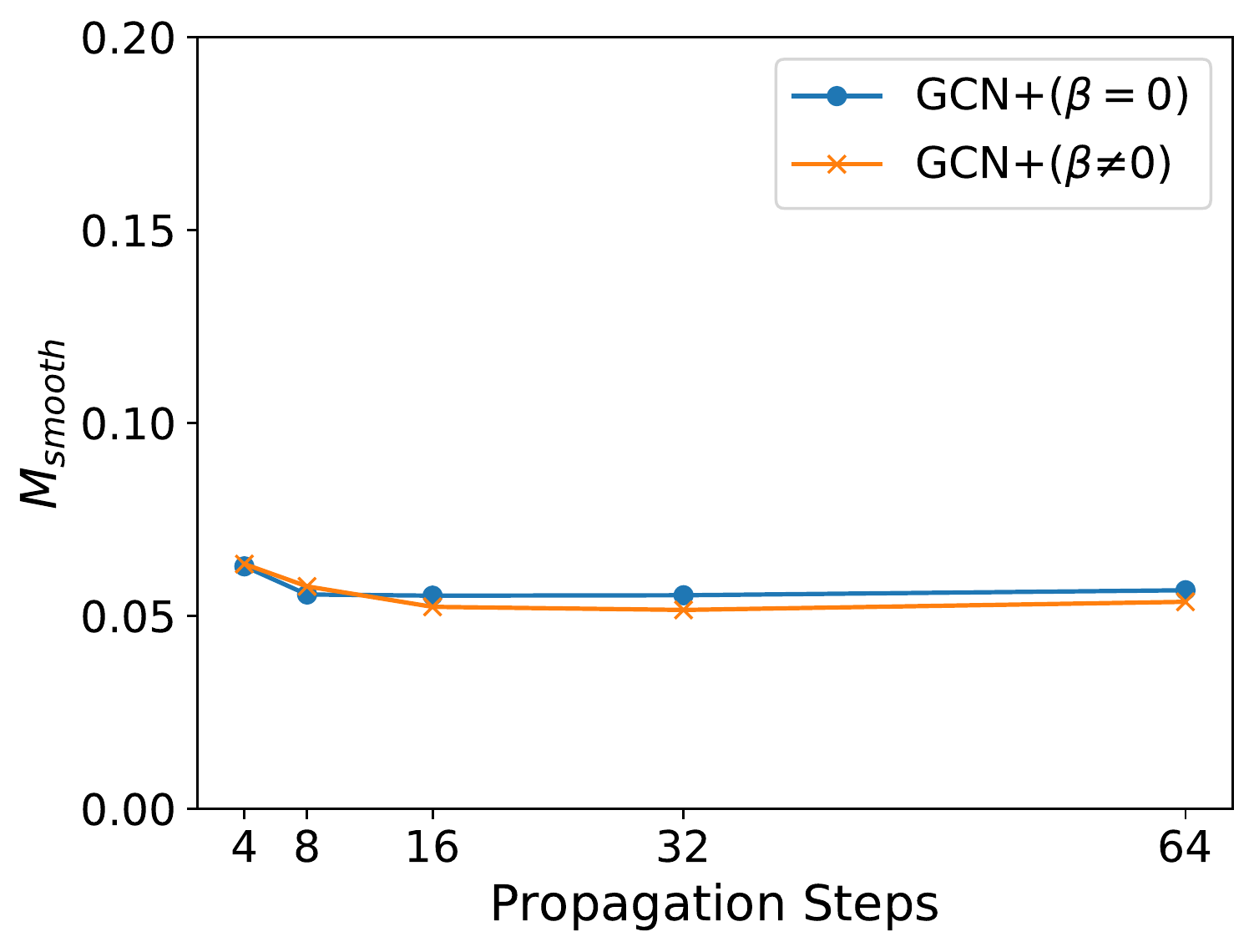} 
	\caption{$M_{\textit{\!non-smooth}}$ of GCN+ with increasing hops on \textit{Cora}.}
	\label{oversmooth}
\end{figure}
\begin{figure}[t]
	\centering
	\includegraphics[width=0.9\columnwidth]{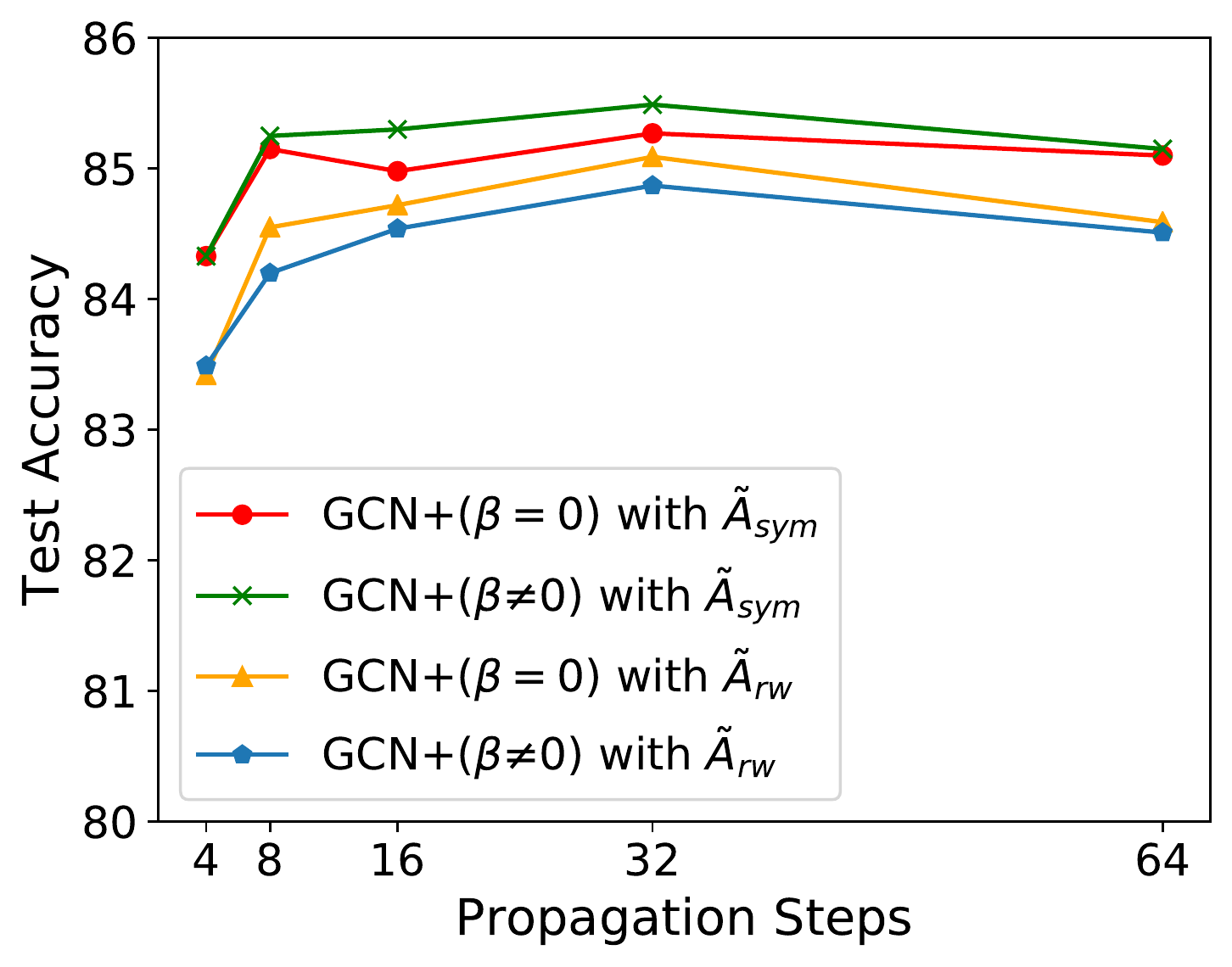} 
	\caption{Performance comparison of different propagation matrices $\tilde{A}_{sym}$ vs. $\tilde{A}_{rw}$ in GCN+ with increasing hops on \textit{Cora}.}
	\label{choice}
\end{figure}
\begin{figure*}[t]
	\centering
	\subfigure[$\alpha=9$]{
		\includegraphics[width=0.24\textwidth]{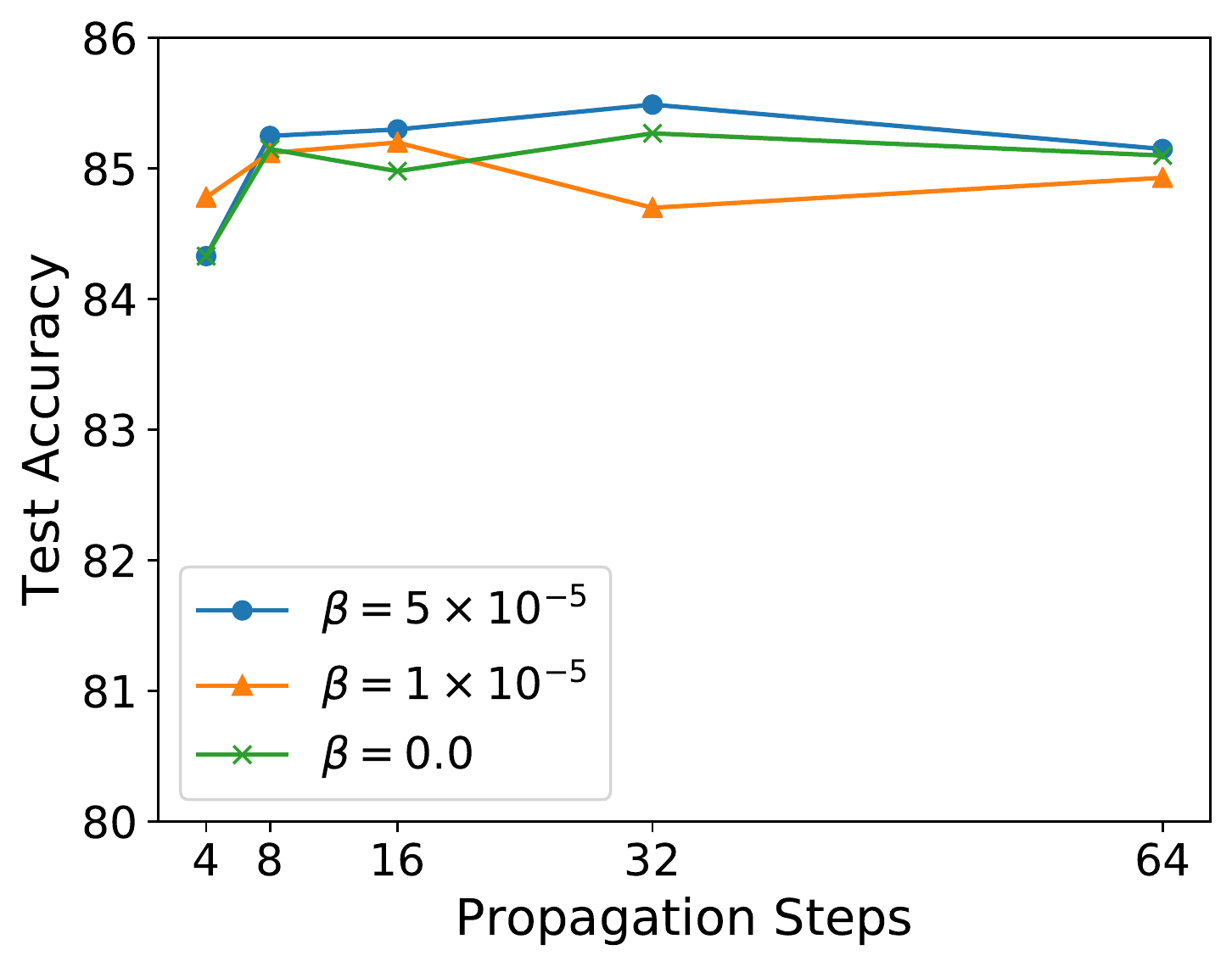}}\hspace{-1mm}
	\subfigure[$\alpha=4$]{
		\includegraphics[width=0.24\textwidth]{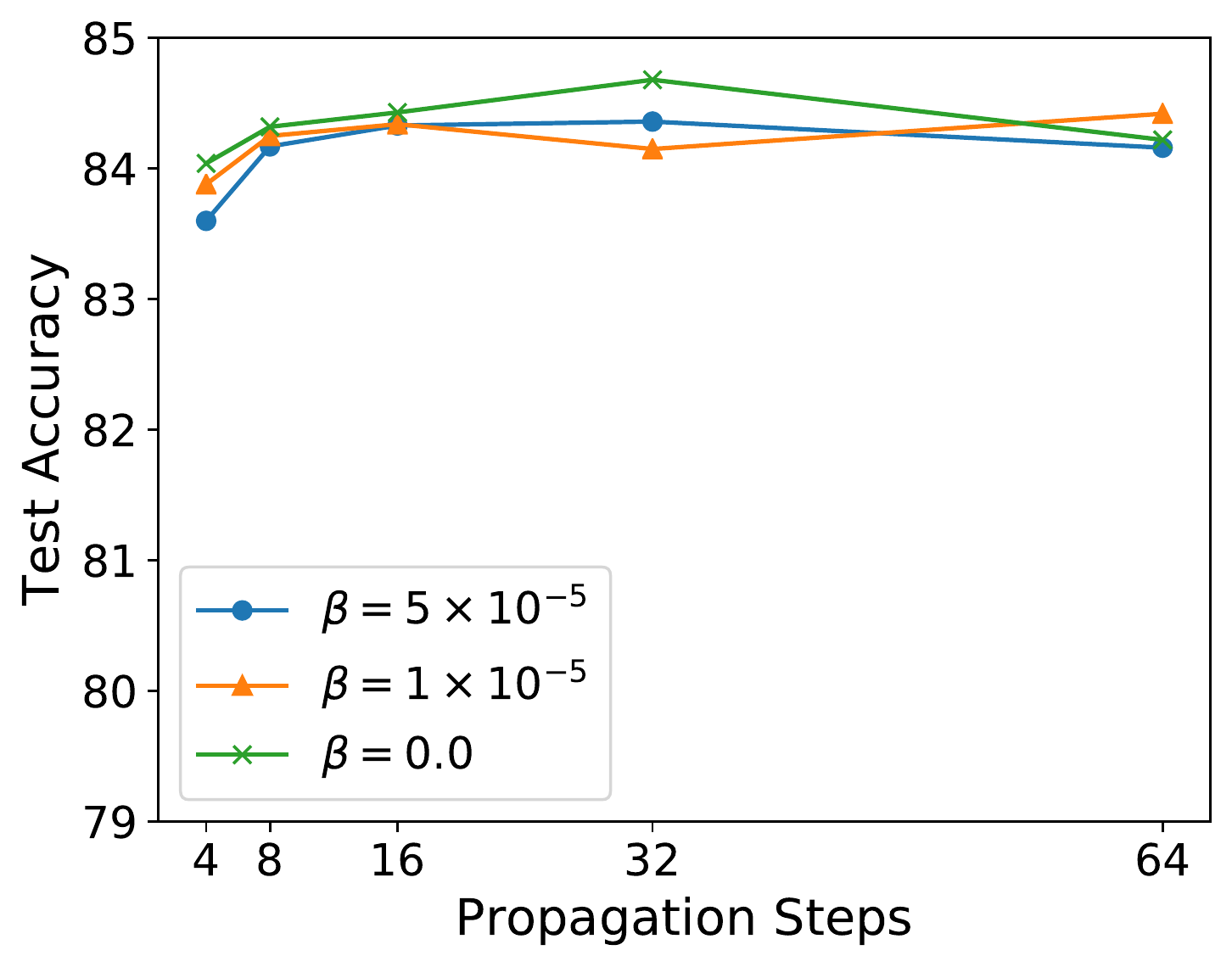}}\hspace{-1mm}
	\subfigure[$\alpha=2$]{
		\includegraphics[width=0.24\textwidth]{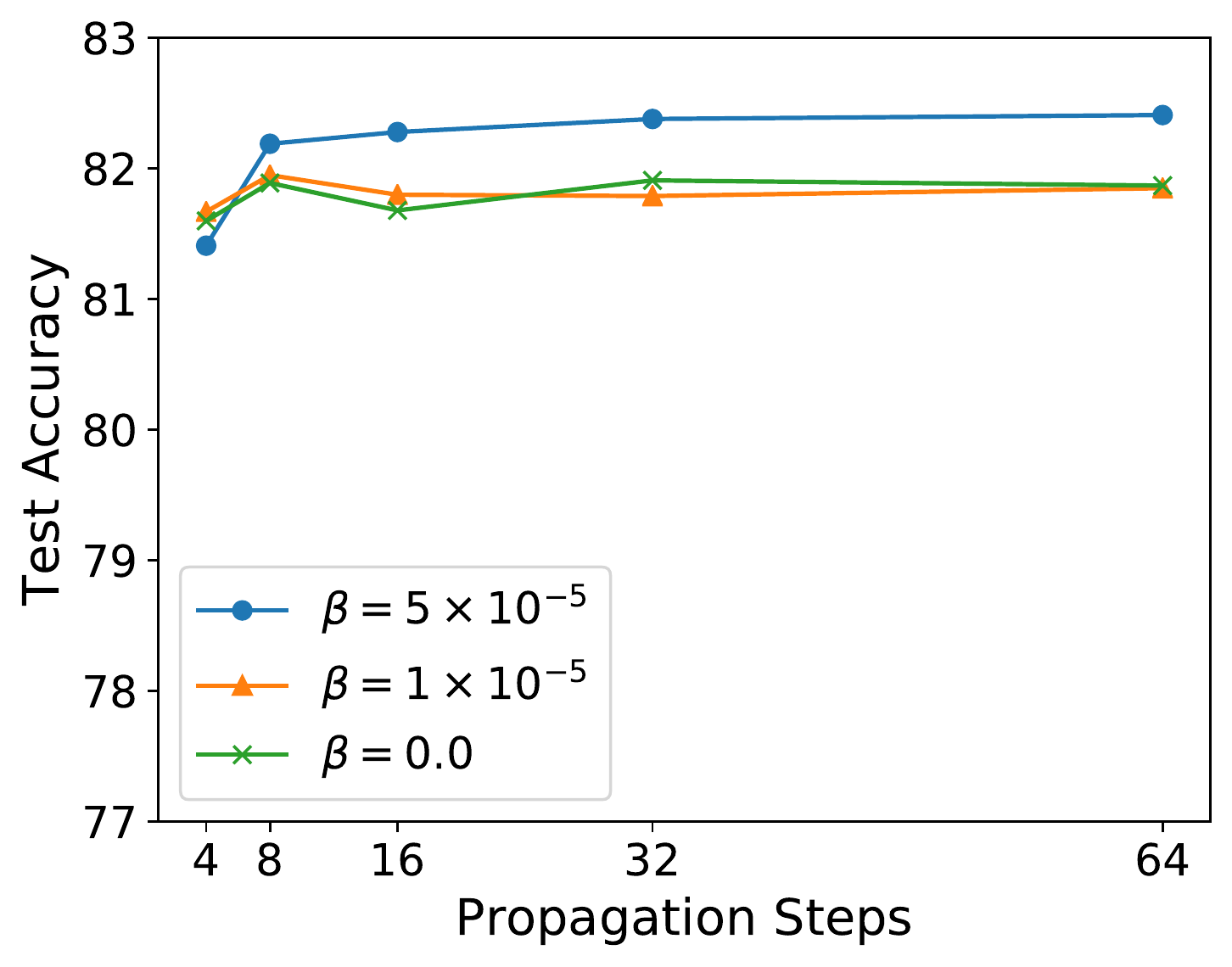}}\hspace{-1mm}
	\subfigure[$\alpha=1$]{
		\includegraphics[width=0.24\textwidth]{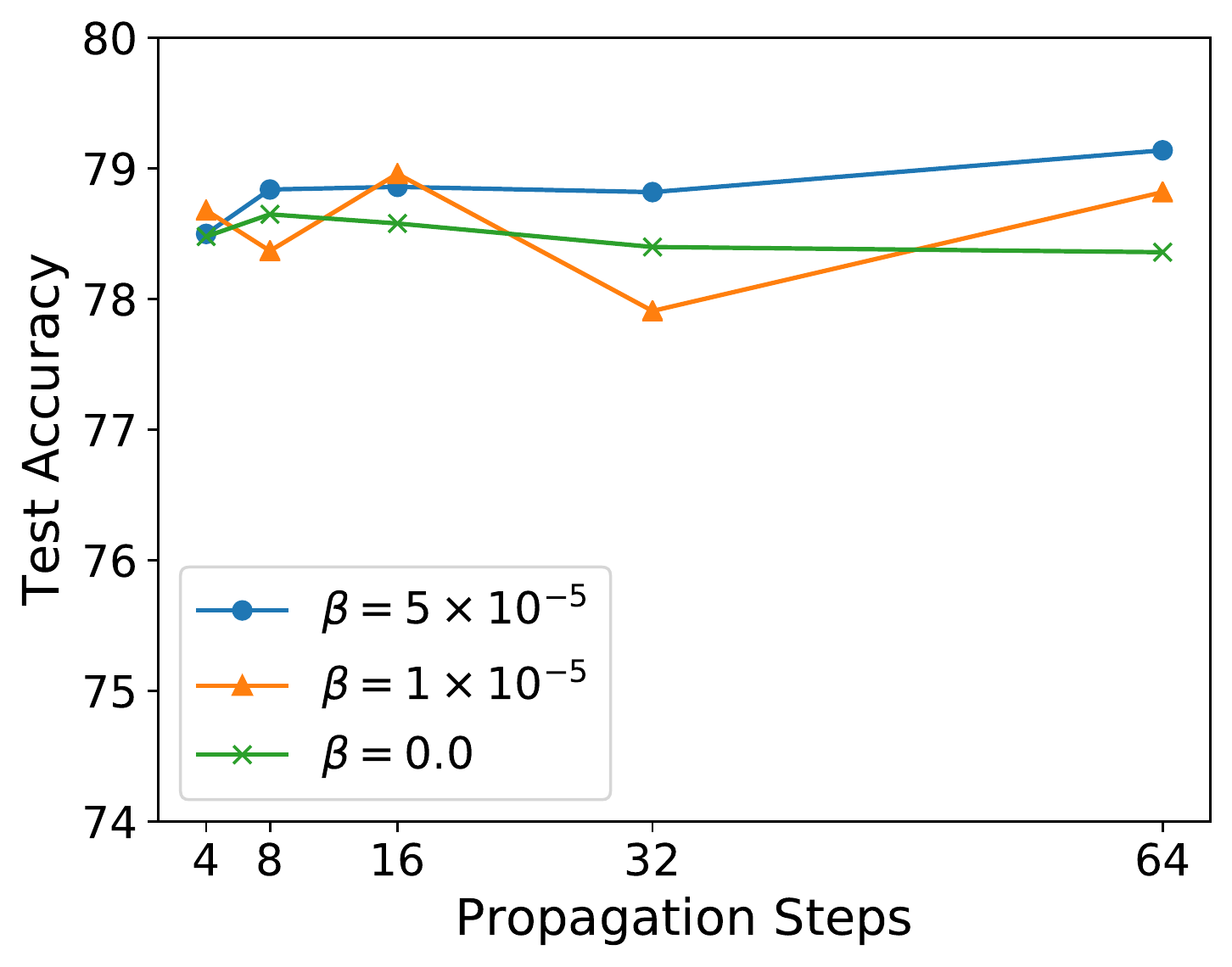}}\hspace{-1mm}
	\caption{Test accuracy of different propagation steps and $\alpha$ on \textit{Cora}.}
	\label{alpha and beta}
\end{figure*}


\subsection{Comparison with SOTA}
The evaluate metric of various datasets are listed in Table \ref{dataset}. Actually it is commonly used to evaluate the model by the community.
\subsubsection{Planetoid}
We use standard fixed and random training/validation/testing splits. Specifically, we use 20 labeled nodes per class as the training set, 500 nodes as the validation set, and 1000 nodes as the test set for all models. For fixed split, we follow the experimental setup in \cite{yang2016revisiting}.  We compare Multiplayer Perception (MLP) ,GCN \cite{kipf2016semi}, GAT \cite{velivckovic2017graph}, SGC \cite{wu2019simplifying}, DAGNN \cite{liu2020towards} and APPNP \cite{klicpera2018predict} with GCN+. Although DropEdge \cite{rong2019dropedge}, PairNorm \cite{zhao2019pairnorm} are proposed to tackle over-smoothing issue recently, our baseline methods don't include them as they do not help to boost the performance on node classification task. 
Table \ref{Plantoid-semi} compares the average test accuracy of 10 runs for each model on Planetoid dataset. As shown, GCN+ outperforms better than the representative baselines. Note that the shallow model APPNP achieves better performance than GCN and GAT. Recent deeper model named DAGNN shows competitive result and robustness on these datasets and GCN+ performs slightly better than it.
\subsubsection{OGB}
We adopt the setting of \cite{hu2020open}, which is more challenging and realistic. We consider the following representative models GCN \cite{kipf2016semi}, GraphSAGE \cite{hamilton2017inductive} and GCNII \cite{chen2020simple} as our baselines. In particular, we use the reported metric of the leaderboards of OGB team, which provide an open benchmark on several tasks and datasets. 

Table \ref{ogb-semi} compares the average test accuracy/ROC-AUC on OGB datasets. As shown, GCN+ outperforms the GCN and GraphSAGE. It is clear that our proposed GCN+ outperform SOTA in two middle scale datasets.

In summary, GCN+ achieves superior performance on several benchmarks, which shows that considering the information of high-order neighbors makes sense and we need more reasonable way to deepen GCNs or make use of the high-order neighbors. Note that GCN+ ($\beta \ne0$) is slightly better than GCN+ ($\beta =0$) which is benefit from the third term of Eq. (\ref{full-opt}). 
\subsection{Over-smoothing Analysis} \label{oversmooth-of-gcn+}
We employ the two proposed metrics to measure the node embeddings learned by GCN+. The results on \textit{Cora} are shown in Fig. \ref{oversmooth}. We can observe that as the number of hops increases, the $M_{\textit{smooth}}$ values nearly remains a small constant which is lower than vanilla deep GCN. This implies that GCN+ use the information of long-range neighbors and do not suffer from over-smoothing. 

Fig. \ref{model_vs} also compares the final output embeddings of GCN+ with multiple hops, which shows different behaviors with GCN. GCN+ relieves the over-smoothing and learns the meaningful embeddings with the increasing hops.

\subsection{Hyperparameter Analysis}
In the previous sections, we use $\tilde{A}$ to refer the $\tilde{A}_{\textit{sym}}$ and $\tilde{A}_{\textit{rw}}$. Here we compare the different choices of propagation matrix $\tilde{A}$. Fig. \ref{choice} depicts the test accuracy achieved by varying the hops of different propagation matrices. The result illustrates that 
$\tilde{A}_{\textit{sym}}$ is slightly better than $\tilde{A}_{\textit{rw}}$. 

We consider three hyperparameter of GCN+, that is $\alpha$, $\beta$ and number of power iteration steps $k$. Fig. \ref{alpha and beta} compares the effect of these hyperparameters on \textit{Cora}. We can see that $k=16,32$ is suitable and more steps does not boost the performance significantly. For \textit{Cora}, when $\alpha=9$ (that means the fraction of retained initial node features is 0.1.), GCN+ achieve the best performance. The value of $\alpha$ varies by different datasets. More results and details listed in the supplementary material.

\section{Related Work}
\subsection{Graph Neural Networks}
Graph neural networks (GNNs) have been extensively studied for the past years. There are different views on designing new architecture, including the spectral-based, spatial-based and other types, such as understand the GNN using dynamic system \cite{xhonneux2019continuous}. Numerous methods are proposed to model the graph-structure data and apply on a wide range of applications. Besides the GCNs, there are also other types of GNNs, such as attention-based GNN \cite{velivckovic2017graph} which use multiple attention to aggregate information from neighbors, autoencoder-based GNN \cite{kipf2016variational}, which use a GCN encoder and decoder to learn meaningful embeddings, and dynamic GNNs \cite{seo2018structured, hajiramezanali2019variational, yan2020sgrnn} which learn the node embedding over time.
\subsection{Deep GCN and Over-smoothing}
Most GNNs are shallow models as deep architecture suffers from over-smoothing. Several studies explore deep GCNs. \citet{xu2018representation} introduce Jumping Knowledge Networks, which uses residual connection to combine the output of each layer. \citet{klicpera2018predict} use Personalized PageRank, which consider the information of root node to replace the graph convolution operator to solve the over-smoothing. DropEdge \cite{rong2019dropedge} suggests that randomly removing the edge of original graph impede over-smoothing. PairNorm \cite{zhao2019pairnorm} is another scheme which uses a normalization layer to scale the node features after the convolution layer. \citet{li2019deepgcns} build on ideas from ResNet to train very deep GCNs. \citet{li2020deepergcn} further propose MsgNorm, which boosts the performance on several datasets. \citet{yang2020revisiting} present NodeNorm to scale the node features. \cite{chen2020simple} propose a deep GCN models which use initial residual connection and identity mapping. 

A few work analyzes the cause and behaviors of over-smoothing theoretically. \citet{oono2019graph} investigate the asymptotic behaviors of GCNs as the layer size tends to infinity and reveals the information loss in deep GCNs. \citet{cai2020note} further extend analysis of \cite{oono2019graph} from linear GNNs to the nonlinear architecture.


\section{Conclusion}
We summarize the existing different views on the mechanism of GCNs, which help us understand and design the graph convolutional kernel. We further provide a general optimization framework named GCN+. Based on this framework, we derive two forms of GCN+ and propose two metrics to measure the smoothness of output node representations. Extensive empirical studies on several real-world datasets demonstrate that GCN+ compares favorably to state of the art with a small amount of parameters. For future work, we will consider different optimization objectives which encode the graph structure and node features adaptively. As we do not limit the transformation from $X$ to $\bar{X}$, another reasonable formulas can be further explored.

\bibliography{refer.bib}

\end{document}